\documentclass{article}

\usepackage[nonatbib,preprint]{neurips_2026}

\usepackage{graphicx}
\usepackage{subcaption}
\usepackage{multirow}
\usepackage[table]{xcolor}

\definecolor{headerGray}{RGB}{245,245,245}
\usepackage[utf8]{inputenc} %
\usepackage[T1]{fontenc}    %
\usepackage{hyperref}       %
\usepackage{url}            %
\usepackage{booktabs}       %
\usepackage{amsfonts}       %
\usepackage{nicefrac}       %
\usepackage{microtype}      %
\usepackage{xcolor}         %
\usepackage{amsmath}
\usepackage{amssymb}
\usepackage{mathtools}
\usepackage{amsthm}

\usepackage{amsmath,amsfonts,bm}

\def\eqref#1{equation~\ref{#1}}

\def\1{\bm{1}}

\DeclareMathAlphabet{\mathsfit}{\encodingdefault}{\sfdefault}{m}{sl}
\SetMathAlphabet{\mathsfit}{bold}{\encodingdefault}{\sfdefault}{bx}{n}

\DeclareMathOperator*{\argmin}{arg\,min}

\usepackage{subcaption,booktabs,multirow}
\usepackage{adjustbox}
\usepackage{hyperref}
\usepackage{url}
\usepackage{url}
\usepackage{amsthm}
\usepackage{listings}
\usepackage{algorithm}
\usepackage{multirow}
\usepackage{algorithmic}
\usepackage{hyperref}       %
\usepackage{url}   
\usepackage{booktabs}
\usepackage{amsfonts}
\usepackage{nicefrac}  
\usepackage{comment} %
\usepackage[numbers]{natbib}
\usepackage{multirow}
\usepackage{float}
\usepackage{booktabs,siunitx,threeparttable,makecell,xcolor,colortbl}
\usepackage[colorinlistoftodos,prependcaption,textsize=small]{todonotes}

\newcommand{\se}[1]{{\scriptsize\textcolor{gray}{#1}}}

\newtheorem{thm}{Theorem}[section]

\usepackage{cleveref}
\crefname{thm}{Theorem}{Theorems}

\theoremstyle{plain}
\newtheorem{theorem}{Theorem}[section]

\newtheorem{lemma}[theorem]{Lemma}

\theoremstyle{definition}

\newtheorem{assumption}[theorem]{Assumption}
\theoremstyle{remark}
\newtheorem{remark}[theorem]{Remark}

\title{MASS-DPO: Multi-negative Active Sample Selection for Direct Policy Optimization}

\author{
Rohan Surana$^{1}$, Xintong Li$^{1}$, Sheldon Yu$^{1}$, Yiran Jenny Shen$^{1}$, Chuhan Wang$^{1}$, Tong Yu$^{2}$, \\
\textbf{Prithviraj Ammanabrolu$^{1}$, Jingbo Shang$^{1}$, Julian McAuley$^{1}$, Junda Wu$^{1}$}
\\
$^{1}$UC San Diego \quad
$^{2}$Adobe Research \\
\texttt{\{rsurana,xil240,ziy040,jes038,chw136,prithvi,jshang,jmcauley,juw069\}@ucsd.edu} \\
\texttt{tyu@adobe.com} \\
}

\begin{document}

\maketitle

\begin{abstract}

Multi-negative preference optimization under the Plackett--Luce (PL) model extends Direct Preference Optimization (DPO) by leveraging comparative signals across one preferred and multiple rejected responses. However, optimizing over large negative pools is costly, and many candidates contribute redundant gradients due to their similar effects on policy updates. 
We introduce MASS-DPO, a multi-negative active sample selection method that derives a PL-specific Fisher-information objective for selecting compact, informative negative subsets within each prompt. The resulting log-determinant objective selects negatives that contribute complementary information for policy updates, yielding compact subsets that retain the full pool's information while reducing redundancy. In practice, this favors negatives whose gradients cover different update directions, reducing redundant signal from near-duplicate candidates while preserving the most useful training information.
Across four benchmarks spanning recommendation and multiple-choice QA and three model families, MASS-DPO consistently exceeds or matches existing methods in accuracy, improves Recall/NDCG and margin-based optimization dynamics, and delivers stronger alignment with substantially fewer negatives.

\end{abstract}

\section{Introduction}

Direct Preference Optimization (DPO)~\citep{rafailov2024direct} aligns models with human preferences by optimizing pairwise comparisons without constructing reward functions~\citep{10.5555/3294996.3295184, ouyang2022training, stiennon2020learning}. Recent work generalizes DPO with the Plackett--Luce (PL) model~\citep{plackett1975analysis,luce1959individual,huang2026listwise,wucontext,xiesurvey} to compare one preferred response against multiple rejected responses, providing richer supervision. 
However, current multi-negative approaches such as Softmax-DPO (S-DPO)~\citep{chen2024softmax} and Direct Multi-Preference Optimization (DMPO)~\citep{bai2024finetuning} typically sample or weight negatives randomly or heuristically. In large candidate pools, this can devote much of the training signal to near-duplicate negatives whose gradients point in similar directions, increasing computation without proportionally improving policy updates.

To address this bottleneck, we propose MASS-DPO (\textbf{M}ulti-negative \textbf{A}ctive \textbf{S}ample \textbf{S}election for \textbf{D}irect \textbf{P}reference \textbf{O}ptimization), an active negative selection framework derived from the multi-negative PL preference objective.
MASS-DPO formulates negative selection as a D-optimal design problem~\citep{pukelsheim2006optimal,kiefer1959optimum}, using a PL-specific Fisher-information objective to measure how much each candidate contributes to policy estimation~\citep{doi:10.1098/rsta.1922.0009,chaloner1995bayesian,flaherty2005robust,kirsch2019batchbald}. We favor D-optimality over alternatives such as A- or E-optimality because maximizing log-determinant information minimizes the volume of the joint confidence ellipsoid, promoting coverage across parameter directions rather than emphasizing a single mode~\citep{kiefer1959optimum,pukelsheim2006optimal}.
Without careful selection, the model can repeatedly update along already-covered directions, leading to poor parameter coverage and inefficient optimization.
MASS-DPO addresses this by selecting negatives that span complementary directions in parameter space, as determined by the D-optimal design formulation.

While D-optimal design provides a principled Fisher-information criterion for prioritizing preference data, most prior work applies it at the \emph{instance level}: selecting which preference samples, prompts, query distributions, or annotators to acquire or retain~\cite{kveton2025active,liu2024dual,das2024active,mukherjee2024optimal}.
MASS-DPO instead applies optimal design \emph{within each prompt}. Given a shared pool of candidate negatives for a preferred response, it selects a compact subset tailored to the multi-negative PL objective using our curvature/Fisher characterization~(\Cref{lemma:hessian}).

The resulting subset selection problem is combinatorial when the candidate pool is large~\citep{krause2012near,kirsch2019batchbald,ni2026survey,huang2025image}. We make it practical with an incremental rank-one procedure that builds the subset one negative at a time using marginal log-determinant gains. Sherman--Morrison updates avoid repeated determinant/inverse recomputation, making the log-determinant objective efficient to optimize~\citep{sener2017active,kirsch2019batchbald,kveton2025active,mundada2026ws}.

Empirically, we show that MASS-DPO improves optimization efficiency and downstream performance across three model families and four recommendation/QA benchmarks, while using substantially fewer negatives.
We summarize our contributions as follows:
\begin{itemize}

    \item We introduce MASS-DPO, a within-prompt active negative selection framework for multi-negative DPO, derived from the Plackett--Luce objective and its Fisher/curvature structure.

    \item We provide an incremental rank-one selection algorithm for efficient log-determinant optimization and establish finite-sample relative-logit error bounds.    
    \item Empirically, MASS-DPO improves optimization efficiency and downstream performance across three language model families and four recommendation/QA benchmarks.
\end{itemize}

\begin{figure}[htp]
\centering
\includegraphics[width=\textwidth]{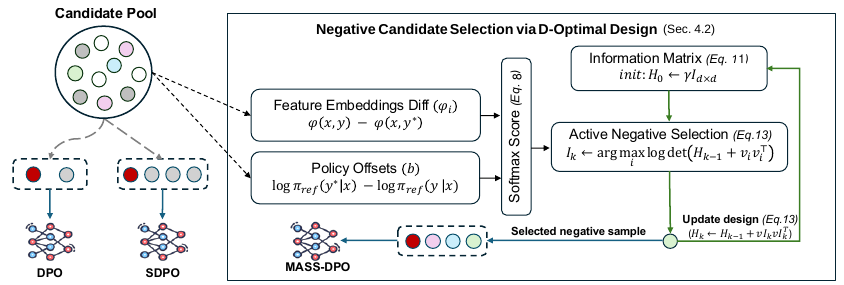}
\caption{Overview of MASS-DPO's D-optimal selection. Each candidate is scored using the feature difference $\phi_i=\phi(x,y_i)-\phi(x,y^*)$ and policy offset $b_i=\log\pi_{\rm ref}(y^*\mid x)-\log\pi_{\rm ref}(y_i\mid x)$, with softmax weights defined in~\Cref{eq:softmax-weights}. The \textcolor{green}{green loop} denotes the subset-construction step in~\Cref{alg:dopt}: starting from $H_0$, we incrementally pick the negative that maximally increases $\log\det H$, then update $H$ accordingly until $n$ samples are selected.}
\label{fig:fig1}

\end{figure}
\section{Related Work}
\textbf{Direct Preference Optimization. }
DPO~\citep{rafailov2024direct} aligns language models with human preferences by optimizing likelihood ratios of preferred over dispreferred responses, avoiding explicit reward modeling and associated complexities such as reward misgeneralization in RLHF~\citep{10.5555/3294996.3295184, ouyang2022training, stiennon2020learning,surana2026generate,huang2026amps,wu2025ocean}. Recent extensions include dynamic margins (ODPO;~\citep{amini2024direct}) and prefix sharing for computational efficiency~\citep{wang2024accelerating}. However, standard DPO is restricted to binary preference pairs, limiting the diversity of supervision~\citep{wang2026scenealign}. Our approach extends beyond binary comparisons by leveraging actively selected, informative multi-negative samples.

\textbf{Multi-negative Preference Optimization. }
Recent work has extended standard DPO's binary preference pairs to leverage multiple negatives for richer comparative signals and enhanced alignment. Softmax-DPO (S-DPO)~\citep{chen2024softmax} generalizes the pairwise Bradley–Terry loss~\citep{19ff28b9-64f9-3656-ba40-08326a05748e} to Plackett–Luce ranking~\citep{plackett1975analysis,luce1959individual,huang2025pluralistic}, providing richer gradient signals. Direct Multi-Preference Optimization (DMPO)~\citep{bai2024finetuning} averages over multiple negatives to promote diverse negative learning. Multi Pair-wise Preference Optimization (MPPO)~\citep{xie2024mppo} extends DPO by directly modeling multi-negative feedback with average-likelihood loss, removing the need for a reference model and enabling flexible use of negative samples. Tree Preference Optimization (TPO)~\citep{liao2024tpo} structures multi-negative alignment through hierarchical preference decomposition. Despite these advances in multi-negative preference optimization, current methods still largely depend on heuristic or random negative selection strategies. Our work addresses this limitation by proposing MASS-DPO, which leverages D-optimal design for theoretically grounded, strategic negative sample selection.

\textbf{Information-theoretic sample selection and optimal design.}
A broad literature in optimal experimental design and active learning selects informative data by maximizing information about model parameters, with D-optimality (maximizing $\log\det$ of the Fisher information) as a standard criterion~\citep{kiefer1959optimum,chaloner1995bayesian,pukelsheim2006optimal}. In modern batch active learning, related objectives are used to promote coverage/diversity in representation or Fisher/gradient space~\citep{sener2017active,kirsch2019batchbald,ash2019deep}. In preference optimization, recent work adopts such principles primarily at the \emph{instance level}---selecting which prompts/comparisons (and, in some settings, teachers) to acquire or retain for training~\citep{kveton2025active,liu2024dual,das2024active,mukherjee2024optimal,huang2025traceable,wu2023infoprompt}.
MASS-DPO instead applies D-optimal design \emph{within each prompt}: given a shared pool of negative candidates, we select a small subset tailored to the \emph{multi-negative} Plackett--Luce objective via our PL-specific curvature/Fisher characterization.
Unlike online hard-negative mining and dynamic sampling strategies~\citep{han-etal-2024-efficient,zhan2021optimizingdenseretrievalmodel,fan2023neighborhood,ma2024negative,li2026importance}, which recompute candidates at every training step and often rely on task-specific heuristics, MASS-DPO operates in a fixed-budget regime: negatives are selected once as a preprocessing step and remain fixed during training, incurring no per-step mining cost. The Fisher-information criterion selects negatives spanning complementary directions in parameter space, a geometric property that persists across training (\Cref{tab:selection-stability}).

\section{Preliminaries}

\subsection{Direct Preference Optimization}

Direct Preference Optimization (DPO)~\citep{rafailov2024direct} aligns a learned policy with human pairwise judgments~\citep{10.5555/3294996.3295184, 10.5555/3495724.3495977, ouyang2022training} without an explicit reward model. 
Under the Bradley-Terry-Luce framework~\citep{19ff28b9-64f9-3656-ba40-08326a05748e}, two responses $y_1,y_2$ to prompt~$x$ with latent scores $r(x,y_1),r(x,y_2)$ satisfy
\begin{equation}
\label{eq:dpo-pref}
p^*(y_1\succ y_2\mid x) = \sigma(r(x,y_1)-r(x,y_2)),
\end{equation}
where $\sigma(z)=1/(1+e^{-z})$. 
Rearranging the optimal-policy relation gives the implicit reward representation up to an \(x\)-dependent additive constant \(\beta \log \mathcal Z(x)\):
\begin{equation}
\label{eq:dpo-adv}
\begin{aligned}
r(x,y)
&= \beta\,\log\!\frac{\pi^*(y\mid x)}{\pi_{\mathrm{ref}}(y\mid x)}
   + \beta\log \mathcal{Z}(x), \\
\mathcal{Z}(x)
&= \sum_{y'}\pi_{\mathrm{ref}}(y'\mid x)\ \cdot\exp\!\Bigl(\tfrac{1}{\beta}\,r(x,y')\Bigr).
\end{aligned}
\end{equation}

Substituting \eqref{eq:dpo-adv} into \eqref{eq:dpo-pref} and simplifying leads to the DPO training objective
\begin{equation}
\begin{aligned}
\mathcal{L}_{\mathrm{DPO}}(\theta)
= -\mathbb{E}_{(x,y_1,y_2)\sim D}\Bigl[
\log\sigma\Bigl(
&\beta\Bigl[
\log\!\tfrac{\pi_\theta(y_1\mid x)}{\pi_{\mathrm{ref}}(y_1\mid x)}
- \log\!\tfrac{\pi_\theta(y_2\mid x)}{\pi_{\mathrm{ref}}(y_2\mid x)}
\Bigr]
\Bigr)\Bigr].
\end{aligned}
\end{equation}

\subsection{Multi-negative Preference Optimization}  
Multi-negative preference optimization generalizes the Direct Preference Optimization framework~\citep{rafailov2024direct} to better align language models with multiple negative preferences.
While traditional DPO employs the Bradley-Terry (BT) model~\citep{19ff28b9-64f9-3656-ba40-08326a05748e} to capture pairwise comparisons, 
multi-negative preference optimization leverages the Plackett-Luce (PL) model~\citep{plackett1975analysis,luce1959individual} to accommodate the ranking of a preferred item against multiple disfavored items.

Consider a user prompt $x_u$ that is formed from historical interactions, along with a preferred item $e_p$ and a set of dispreferred items $E_d$. 
The aim is to maximize the probability that the preferred item $e_p$ is ranked above every item in $E_d$, as described by  
\begin{equation} \label{eq:pl}
p^*(e_p \succ E_d \mid x_u) = \frac{\exp(r(x_u, e_p))}{\sum_{e_d \in \{e_p\} \cup E_d} \exp(r(x_u, e_d))},
\end{equation}  
where $r(x_u, e)$ is the latent reward function defined over the prompt-response pairs in the RLHF framework~\citep{ouyang2022training}. 
From \Cref{eq:pl}, we obtain the following multi-negative preference loss:
\begin{equation}
\label{eq:dpo-loss}
\begin{aligned}
\mathcal{L}(\theta)
&= -\mathbb{E}_{(x_u,e_p,E_d)\sim D}\Bigl[
\log \sigma\Bigl(
-\log \sum_{e_d\in E_d}\exp\bigl(\beta\,\Delta(x_u,e_d,e_p)\bigr)
\Bigr)\Bigr]
\end{aligned}
\end{equation}
with $\sigma(\cdot)$ denoting the sigmoid function and $
\Delta(x_u,e_d,e_p) =
\log\!\frac{\pi_\theta(e_d \mid x_u)}{\pi_{\mathrm{ref}}(e_d \mid x_u)}
- \log\!\frac{\pi_\theta(e_p \mid x_u)}{\pi_{\mathrm{ref}}(e_p \mid x_u)}$. When $|E_d|=1$, this reduces to the standard pairwise DPO objective.

\section{MASS-DPO: Multi-negative Active Sample Selection}
\label{sec:mdpo}

In multi-negative preference optimization tasks (\emph{e.g.}, recommendation, multiple-choice QA, information retrieval), the selection of negative samples significantly influences alignment efficiency and effectiveness. \emph{Uninformative} negatives, already well-separated from preferred responses, waste gradient computations and hinder convergence~\citep{YANG2023119155,kalantidis2020hard,robinson2020contrastive,zhang2022false}. Thus, the key challenge is strategically selecting a compact yet informative subset of negatives to highlight the policy’s weaknesses while maintaining numerical stability~\citep{ma2024negative,kirsch2022unifying,fan2023neighborhood}.
To address this, we propose MASS-DPO~(\Cref{fig:fig1}), an active negative selection method formulated as a D-optimal design problem~\citep{pukelsheim2006optimal,cohn1993neural,kirsch2019batchbald}, maximizing a Fisher-information objective~\citep{doi:10.1098/rsta.1922.0009,jung2021optimal,liu2024dual,neilsen2018optimal,sourati2017asymptotic,chaloner1995bayesian,ash2021gone}. By maximizing this objective, MASS-DPO minimizes the volume of the confidence ellipsoid of policy parameters, connecting computational efficiency with statistical guarantees (Section~\ref{sec:theory}).

\begin{algorithm}[htp]
\small
\caption{D‑Optimal Multi‑negative Active Sample Selection}
\label{alg:dopt}
\begin{algorithmic}[1]
\STATE \textbf{Input:} context $x$, preferred response $y^*$, candidate set $\mathcal C=\{y_i\}_{i=1}^N$, preprocessing parameter $\theta_0$, scale $\beta$, ridge $\gamma$, number of negatives $n$
\STATE Compute feature differences and offsets, for each $i\in[N]$, \\
  \quad $\phi_i \leftarrow \phi(x,y_i) - \phi(x,y^*)$, \\
  \quad $b_i \leftarrow \log\pi_{\mathrm{ref}}(y^*\mid x) - \log\pi_{\mathrm{ref}}(y_i\mid x)$
\STATE Compute scores and softmax weights, for each $i\in[N]$, \\
  \quad $s_i \leftarrow \beta\bigl(\phi_i^\top\theta_0 + b_i\bigr)$, \\
  \quad $q_i^0 \leftarrow \exp(s_i)/\sum_{k=1}^N\exp(s_k)$
\STATE Center and weight features, for all $i\in[N]$ \\
  \quad $\bar\phi_0 \leftarrow \sum_{j=1}^N q_j^0\,\phi_j$,
  \quad $\tilde\phi_i^0 \leftarrow \phi_i - \bar\phi_0$,
  \quad $v_i^0 \leftarrow \sqrt{q_i^0}\,\tilde\phi_i^0$
\STATE Compute fixed Fisher scale: \\
  \quad $Z_{\mathcal C}^0 \leftarrow -\log\sum_{i=1}^N\exp(s_i)$,
  \quad $\alpha_0 \leftarrow \beta^2(1-\sigma(Z_{\mathcal C}^0))$
\STATE Initialize matrices and selected index set: $H_0 \leftarrow \gamma\mathbf{I}_{d\times d}$, $I_0 \leftarrow \emptyset$
\FOR{$k = 1,\dots,n$}
  \STATE Select index: \\
  \quad $i_k \leftarrow \arg\max_{i\in [N]\setminus I_{k-1}} \log\det\!\bigl(H_{k-1} + \alpha_0 v_i^0(v_i^0)^\top\bigr)$
\STATE Update selected indices and design matrix: \\
\quad $I_k \leftarrow I_{k-1} \cup \{i_k\}$,\quad $H_k \leftarrow H_{k-1} + \alpha_0 v_{i_k}^0(v_{i_k}^0)^\top$
\ENDFOR
\STATE \textbf{Output:} selected negatives set $S_n=\{y_i:i\in I_n\}$
\end{algorithmic}
\end{algorithm}

\subsection{Setting}\label{sec:setting}
Following prior work in regret minimization and reward-model active learning~\citep{kveton2025active, riquelme2018deep, das2024active,mukherjee2024optimal, liu2024dual, thekumparampil2024comparing}, we adopt a log-linear policy model for the selection criterion. We assume:
\begin{assumption} 
\label{assump:log-linear}
We assume the policy under consideration takes a log-linear form:
\begin{equation}
\pi(y \mid x; \theta) \propto \exp\bigl(\phi(x, y)^\top \theta\bigr),
\end{equation}
where $\phi(x,y) \in \mathbb{R}^d$ denotes the feature embedding of the context-response pair $(x,y)$, and $\theta \in \mathbb{R}^d$ the model parameters.
\end{assumption}
We now specialize the multi-negative loss $\mathcal{L}$ from~\Cref{eq:dpo-loss} to a single prompt $x$ with preferred response $y^*$ and candidate negatives $\mathcal C=\{y_i\}_{i=1}^N$. Under Assumption~\ref{assump:log-linear}, defining the feature difference $\phi_i = \phi(x,y_i) - \phi(x,y^*)$ and reference-policy offset $b_i = \log\frac{\pi_{\mathrm{ref}}(y^*\mid x)}{\pi_{\mathrm{ref}}(y_i\mid x)}$ for each negative $y_i$ relative to the preferred response $y^*$, the multi-negative DPO loss takes the compact form:
\begin{equation}
\label{eq:sdpo-loss}
L(\theta;S_n) = -\log\sigma\Bigl(-\log\sum_{i\in S_n} \exp\bigl(\beta\,(\phi_i^\top \theta + b_i)\bigr)\Bigr),
\end{equation}
where $S_n\subseteq \mathcal C$ is a subset of size $n$ drawn from the candidate pool. Our goal is to choose $S_n$ so as to maximize the information it provides about $\theta$. The following lemmas quantify how each candidate negative alters the gradient and curvature, showing that negatives with \emph{diverse} and \emph{orthogonal} feature differences enlarge the information matrix the most, while redundant examples leave its volume almost unchanged.

\begin{lemma}[Gradient of Multi-negative Loss]
\label{lemma:gradient}
Define the normalization factor $Z_n$ and subset-normalized softmax weights $q_j^{S_n}(\theta)$ as
\begin{equation}
\label{eq:softmax-weights}
\begin{aligned}
q_j^{S_n}(\theta)
&= \frac{\exp\bigl[\beta(\phi_j^\top\theta + b_j)\bigr]}
{\sum_{k\in S_n}\exp\bigl[\beta(\phi_k^\top\theta + b_k)\bigr]}, \\
Z_n
&= -\log\sum_{i\in S_n}\exp\Bigl[\beta(\phi_i^\top\theta + b_i)\Bigr].
\end{aligned}
\end{equation}
Then the gradient of \eqref{eq:sdpo-loss} with respect to $\theta$ is given by
\begin{equation}
\label{eq:grad}
\nabla_\theta L(\theta;S_n) = \beta\,(1 - \sigma(Z_n))\sum_{j\in S_n} q_j^{S_n}(\theta) \,\phi_j.
\end{equation}
\end{lemma}
The detailed derivation is in Appendix~\ref{app:gradient_derivation}.
The gradient is a weighted combination of feature differences scaled by the misranking probability $(1-\sigma(Z_n))$. At training time, the subset weights $q_j^{S_n}(\theta)$ emphasize negatives with small score margins, indicating that the highest-leverage gradient directions correspond to borderline, hard-to-rank candidates.

\begin{lemma}[Hessian and Curvature]
\label{lemma:hessian}
Let $\bar\phi_{S_n}(\theta) = \sum_{j\in S_n} q_j^{S_n}(\theta)\,\phi_j$ denote the expected feature difference under the subset softmax distribution. The Hessian of~\Cref{eq:sdpo-loss} is then
\begin{align}
\nabla^2L(\theta;S_n) &= \beta^2(1-\sigma(Z_n))\Bigl[\sigma(Z_n)\,\bar\phi_{S_n}\bar\phi_{S_n}^\top + \sum_{j\in S_n} q_j^{S_n}(\theta)(\phi_j - \bar\phi_{S_n})(\phi_j - \bar\phi_{S_n})^\top\Bigr] \nonumber \\
&\succeq \beta^2(1-\sigma(Z_n))\sum_{j\in S_n} q_j^{S_n}(\theta)(\phi_j - \bar\phi_{S_n})(\phi_j - \bar\phi_{S_n})^\top.
\label{eq:hessian_lb}
\end{align}
The inequality follows because $\sigma(Z_n)\,\bar\phi_{S_n}\bar\phi_{S_n}^\top\succeq 0$; dropping it yields a Loewner lower bound that isolates the dispersion of feature differences around their mean. The full derivation is in Appendix~\ref{app:hessian_derivation}.
\end{lemma}

The Hessian lower bound motivates maximizing the determinant of the weighted covariance: subsets whose feature differences spread along orthogonal directions yield the largest information volume, providing a natural selection criterion.

\subsection{Negative Selection via D‑Optimal Design}

While a larger negative pool can in principle improve parameter estimates, many candidates contribute redundant information already conveyed by a smaller, well-chosen subset.
MASS-DPO casts negative selection as a \emph{D-optimal design}~\citep{kiefer1959optimum,pukelsheim2006optimal,kirsch2019batchbald} problem that maximizes the information gain~\citep{chaloner1995bayesian} about the policy parameters.

\noindent\textbf{Fisher-information objective.}\;
Following standard practice in D-optimal active learning~\citep{chaloner1995bayesian,kveton2025active}, Algorithm~\ref{alg:dopt} evaluates the design at a fixed reference point: the full-pool weight $q_j^0$ and center $\bar\phi_0=\sum_{j\in\mathcal C}q_j^0\phi_j$ are computed once before training, defining each candidate's Fisher contribution as $v_j^0=\sqrt{q_j^0}(\phi_j-\bar\phi_0)$ and giving a fixed information matrix $H(S)$ amenable to efficient rank-one optimization.
Given a subset $S\subseteq\mathcal C$ we define the regularized information matrix:
\begin{equation}
\label{eq:fis}
    H(S)
    =\gamma I+ \alpha_0\sum_{j\in S}v_j^0(v_j^0)^{\top}, \quad
    \alpha_0=\beta^2(1-\sigma(Z_{\mathcal C}^0)),\quad \gamma>0,
\end{equation}
where $Z_{\mathcal C}^0=-\log\sum_{i\in\mathcal C}\exp[\beta(\phi_i^\top\theta_0+b_i)]$ is fixed during subset construction.
The ridge $\gamma>0$ ensures $H(S)$ is well conditioned for all subsets.
The D-optimal criterion seeks
\begin{equation}
\label{eq:d-opt}
S_n^{*}
\;=\;
\arg\max_{S\subseteq\mathcal C,\;|S|=n}
\log\det H(S),
\end{equation}
which maximizes the information volume, equivalently minimizing the volume of the confidence ellipsoid for the policy parameters.
However,~\Cref{eq:d-opt} is NP-hard~\citep{7e259dfd-e83e-3fa3-8846-af1d03e58e65,allen2021near} as it requires searching over $\binom{|\mathcal C|}{n}$ subsets.
We therefore build the subset incrementally using marginal log-determinant gains~\citep{05bbe48e-d718-3ac9-8152-8b3d41762a1c,krause2008near}.

\paragraph{Incremental subset construction.}
Starting from $H_0=\gamma I$ and selected index set $I_0=\emptyset$, we build the selected negative set $S_n=\{y_i:i\in I_n\}$ one element at a time via rank-one updates.
At iteration $k$, select the next negative by maximizing the marginal log-determinant gain:
\begin{align}
i_k &\leftarrow \arg\max_{i\in [N]\setminus I_{k-1}}
\log\det\!\bigl(H_{k-1}+\alpha_0 v_i^0 (v_i^0)^\top\bigr), \nonumber\\
I_k &\leftarrow I_{k-1}\cup\{i_k\}, \quad
H_k \leftarrow H_{k-1}+\alpha_0 v_{i_k}^0(v_{i_k}^0)^\top. \label{eq:incremental_update}
\end{align}

Using the matrix determinant lemma,
\begin{align}
\log\det\!\bigl(H_{k-1}+\alpha_0 v_i^0 (v_i^0)^\top\bigr)
&= \log\det H_{k-1} + \log\!\bigl(1+ \alpha_0 (v_i^0)^\top H_{k-1}^{-1} v_i^0\bigr),
\label{eq:mdl_incremental}
\end{align}

so the selection rule is equivalently $i_k=\arg\max_{i\notin I_{k-1}} (v_i^0)^\top H_{k-1}^{-1} v_i^0$ (Alg.~1).
After scoring the remaining candidates, the Sherman--Morrison inverse update costs $\mathcal{O}(d^2)$ per selected negative, and scoring all remaining candidates costs $\mathcal{O}(|\mathcal C|d^2)$ per step.
The score $(v_i^0)^\top H_{k-1}^{-1} v_i^0$ is the $H_{k-1}$-induced squared norm of $v_i^0$; the procedure thus prefers negatives that probe the least-covered directions of the parameter space. We empirically verify that fixed selections remain stable across training~(\Cref{app:imp},~\Cref{tab:selection-stability}).

\section{Theoretical Analysis}
\label{sec:theory}

Having established the D-optimal selection criterion, we now analyze how well a policy trained on the selected subset $S_n$ approximates one trained on the full pool.
Our goal is to bound the relative logit error---the worst-case distortion of pairwise candidate-margin rankings---as a function of subset size $n$ and dimension $d$.
The analysis relies on standard assumptions on feature boundedness, design-weight regularity, and candidate diversity; full statements are deferred to~\Cref{app:assumptions}.

\begin{thm}[Relative Logit Error Bound]\ \\
\label{thm:logit-error}
For a fixed prompt with candidate-negative pool $\mathcal C=\{y_i\}_{i=1}^N$, let $\theta_*$ minimize the regularized full-pool loss and $\hat{\theta}_n$ minimize the regularized loss on the subset $S_n\subseteq\mathcal C$ of size $n$ returned by Algorithm~\ref{alg:dopt}. Define the \emph{relative logit error}
\begin{equation}
\label{eq:logit-error-def}
  \mathcal{E}_{\mathrm{rel}}(\hat{\theta}_n, \theta_*)
  \;=\;
  \max_{i,j \in \mathcal C}
  \bigl|(\phi_i-\phi_j)^\top(\hat{\theta}_n-\theta_*)\bigr|,
\end{equation}
measuring the worst-case distortion of pairwise candidate-margin rankings induced by using a subset-trained estimator in place of the full-pool optimum. Under Assumptions~\ref{assump:boundedness}--\ref{assump:fisher-compat},
\[
\mathcal{E}_{\mathrm{rel}}(\hat\theta_n,\theta_*)
\;\leq\;
\widetilde O\!\left(
d\sqrt{\frac{\log(1/\delta)}{n}}
\right),
\]
with probability at least $1-\delta$, where $\widetilde O$ hides logarithmic factors and candidate-pool regularity constants. The bound decays as $1/\sqrt{n}$ with the selected-negative budget, showing that the subset-trained policy converges to the full-pool policy in relative-logit error as $n$ grows.
\end{thm}
The formal statement with explicit constants, the Fisher-compatibility and estimator-stability conditions, and the proof are given in~\Cref{app:thm-formal}.
\begin{thm}[Batch Design Estimation Error]
\label{thm:batch-error}
With probability at least $1-\delta$, given $k$ prompts each with $n$ selected negatives $S_{k,n}$, the deviation of the regularized batch estimator $\hat\theta_{k,n}$ from the full-pool optimum $\theta_*$ is bounded in the $\Sigma_{k,n}$-norm:
\begin{equation}
\label{eq:batch-error}
\bigl\|\hat\theta_{k,n}-\theta_*\bigr\|_{\Sigma_{k,n}}
\;\leq\;
\sqrt{\frac{d}{4}\log\!\left(\frac{1/\delta + k\,c_{\min}/\gamma}{\bigl(1-c_{\min}\,k/\gamma\bigr)^{1/d}\,\delta}\right)}
\;+\; 2\gamma^{1/2},
\end{equation}
where $\Sigma_{k,n}=\gamma I+\nabla^2 L(\theta_*; S_{k,n})$, and $c_{\min},\gamma,\beta$ are the constants from Assumptions~\ref{assump:boundedness}--\ref{assump:design-matrix}.
\end{thm}

The probability is over i.i.d.\ sampling of $k$ prompts; this follows~\citep{abbasi2011improved,kveton2025active} by treating the multi-negative loss as a generalized linear model and applying self-normalized concentration to the stochastic gradients. In practice, \Cref{thm:logit-error} shows that even with a small selected-negative budget MASS-DPO can already achieve bounded logit error, which translates into faster convergence; \Cref{thm:batch-error} further implies that the selected negatives ensure stable generalization across prompts, which we verify empirically in~\Cref{sec:exp}.
Algorithm~\ref{alg:dopt} optimizes the log-determinant objective via incremental rank-one updates; the connection to relative-logit error is carried by the leverage bound in Appendix~\ref{app:leverage}.

\begin{table*}[ht]
\small
\centering
\caption{Accuracy (\%) on four tasks across three base models. Each entry reports accuracy$_{\text{SE}}$, where the subscript denotes standard error. \textbf{Bold} = best, \underline{underlined} = second best.}
\begin{tabular}{l|l|ccccc}
\toprule
\multirow{1}{*}{Model} & \multirow{1}{*}{Setting} & Medmcqa & QASC & LastFM & MovieLens & Avg \\
\midrule

\multirow{5}{*}{Qwen3}
& DPO & 43.49\se{0.76} & 68.43\se{1.12} & 45.75\se{0.80} & 31.96\se{0.74} & 47.41\se{0.86} \\
& DMPO & 28.91\se{0.72} & 66.78\se{1.12} & 43.40\se{0.78} & 25.66\se{0.69} & 41.19\se{0.83} \\
& DPO-k & \underline{55.56\se{0.77}} & \underline{71.96\se{1.06}} & \underline{51.10\se{0.80}} & 44.56\se{0.77} & \underline{55.80\se{0.85}} \\
& S-DPO & 52.56\se{0.77} & 71.08\se{1.07} & 50.25\se{0.80} & \textbf{48.19\se{0.78}} & 55.52\se{0.86} \\
& MASS-DPO & \textbf{56.66\se{0.77}} & \textbf{72.19\se{1.05}} & \textbf{52.30\se{0.79}} & \underline{47.58\se{0.80}} & \textbf{57.18\se{0.85}} \\
\midrule
\multirow{5}{*}{SmolLM3}
& DPO & 33.27\se{0.73} & 67.00\se{1.09} & 51.90\se{0.80} & 37.60\se{0.77} & 47.44\se{0.85} \\
& DMPO & 25.50\se{0.68} & 65.23\se{1.10} & 50.10\se{0.80} & 28.68\se{0.69} & 42.38\se{0.82} \\
& DPO-k & 44.09\se{0.79} & \underline{69.98\se{1.01}} & 55.70\se{0.78} & 51.36\se{0.79} & 55.28\se{0.84} \\
& S-DPO & \textbf{44.99\se{0.79}} & 69.43\se{1.06} & \underline{55.90\se{0.78}} & \textbf{55.70\se{0.78}} & \underline{56.50\se{0.85}} \\
& MASS-DPO & \underline{44.19\se{0.79}} & \textbf{71.63\se{1.07}} & \textbf{57.25\se{0.79}} & \underline{54.03\se{0.77}} & \textbf{56.78\se{0.85}} \\
\midrule
\multirow{5}{*}{Llama3}
& DPO & 52.25\se{0.80} & 71.08\se{1.04} & 54.60\se{0.82} & 33.52\se{0.75} & 52.86\se{0.85} \\
& DMPO & 25.70\se{0.69} & 69.87\se{1.08} & 49.95\se{0.80} & 28.18\se{0.72} & 43.42\se{0.82} \\
& DPO-k & 71.04\se{0.73} & \underline{73.95\se{0.96}} & 55.65\se{0.80} & 44.46\se{0.77} & 61.27\se{0.82} \\
& S-DPO & \textbf{72.19\se{0.72}} & \textbf{74.61\se{0.97}} & \underline{56.55\se{0.79}} & \underline{49.55\se{0.80}} & \textbf{63.23\se{0.82}} \\
& MASS-DPO & \underline{71.29\se{0.74}} & 73.62\se{1.03} & \textbf{57.35\se{0.81}} & \textbf{49.70\se{0.80}} & \underline{62.99\se{0.84}} \\
\bottomrule
\end{tabular}
\label{table:results}
\end{table*}

\section{Experiments}
\label{sec:exp}

Section~\ref{sec:exp-rq1} isolates MASS-DPO's D-optimal selection criterion by comparing against the random softmax weighting in S-DPO. Section~\ref{sec:exp-rq2} reports downstream accuracy against existing preference-optimization baselines across three backbone families. Section~\ref{sec:exp-rq3} evaluates the quality of the negatives produced by the incremental selection procedure using standard ranking metrics.

\textbf{Datasets.} Following recent DPO-based recommendation work~\citep{chen2024softmax,sun2024direct,he2025reclaif}, we utilize two widely adopted real-world recommendation benchmarks: LastFM~\citep{Bertin-Mahieux2011} and MovieLens~\citep{harper2015movielens}. For QA tasks, we adopt two challenging multiple-choice QA datasets: MedMCQA~\citep{pmlr-v174-pal22a}, a medical-domain QA benchmark, and QASC~\citep{khot2020qasc}, a scientific reasoning QA dataset. These tasks naturally feature large candidate pools with well-defined negatives, providing controlled benchmarks for evaluating active selection strategies. We report Accuracy, Margin, Chosen Rewards, and additional utility metrics; detailed methodology is in Appendix~\ref{app:exp}.

\textbf{Methods.}
We benchmark MASS-DPO against established preference alignment approaches: pairwise DPO~\citep{rafailov2024direct}, the multi-negative extension DPO-k, Softmax-DPO (S-DPO)~\citep{chen2024softmax}, and DMPO~\citep{bai2024finetuning}. To maintain fairness and manage computational costs, the number of negative candidates during training is set to 3 for all multi-negative methods (DPO-k, DMPO, S-DPO, MASS-DPO) and 1 for DPO. At test time, we evaluate against all available candidates (up to 20) to measure the model's ability to rank under a larger search space. Implementation details are provided in~\Cref{app:imp}.

\begin{table*}[ht]
\centering
\caption{Recall (R) and NDCG (N) at k=\{1,3\} on LastFM and MovieLens. Each entry reports metric$_{\text{SE}}$, where the subscript denotes standard error.}
\scriptsize
\setlength{\tabcolsep}{5pt}
\begin{tabular}{l l | cccc | cccc}
\toprule
\multirow{2}{*}{Model} & \multirow{2}{*}{Method} & \multicolumn{4}{c|}{LastFM} & \multicolumn{4}{c}{MovieLens} \\
& & R@1 & R@3 & N@1 & N@3 & R@1 & R@3 & N@1 & N@3 \\
\midrule
\multirow{5}{*}{Qwen3}
& DPO      & 46.15\se{1.11} & 72.60\se{1.00} & 46.15\se{1.11} & 61.60\se{0.93} & 29.64\se{1.03} & 59.48\se{1.10} & 29.64\se{1.03} & 46.89\se{0.95} \\
& DMPO     & 44.50\se{1.11} & 72.05\se{1.00} & 44.50\se{1.11} & 60.57\se{0.93} & 24.50\se{0.97} & 56.30\se{1.11} & 24.50\se{0.97} & 42.88\se{0.92} \\
& DPO-k    & 49.50\se{1.12} & 76.45\se{0.95} & 49.50\se{1.12} & 65.36\se{0.90} & 41.63\se{1.11} & 68.95\se{1.04} & 41.63\se{1.11} & 57.71\se{0.95} \\
& S-DPO    & 48.55\se{1.12} & 75.10\se{0.97} & 48.55\se{1.12} & 64.14\se{0.91} & 45.92\se{1.12} & 71.47\se{1.01} & 45.92\se{1.12} & 60.86\se{0.94} \\
& MASS-DPO & \textbf{51.10}\se{1.12} & \textbf{77.20}\se{0.94} & \textbf{51.10}\se{1.12} & \textbf{66.48}\se{0.90} & \textbf{45.97}\se{1.12} & \textbf{71.52}\se{1.01} & \textbf{45.97}\se{1.12} & \textbf{61.10}\se{0.94} \\
\midrule
\multirow{5}{*}{SmolLM3}
& DPO      & 51.70\se{1.12} & 78.15\se{0.92} & 51.70\se{1.12} & 67.29\se{0.89} & 37.25\se{1.09} & 65.68\se{1.07} & 37.25\se{1.09} & 53.77\se{0.95} \\
& DMPO     & 50.30\se{1.12} & 77.90\se{0.93} & 50.30\se{1.12} & 66.54\se{0.89} & 28.23\se{1.01} & 60.43\se{1.10} & 28.23\se{1.01} & 47.02\se{0.93} \\
& DPO-k    & 56.30\se{1.11} & 80.55\se{0.89} & 56.30\se{1.11} & 70.71\se{0.86} & 51.01\se{1.12} & 75.71\se{0.96} & 51.01\se{1.12} & 65.41\se{0.92} \\
& S-DPO    & 55.60\se{1.11} & \textbf{81.35}\se{0.87} & 55.60\se{1.11} & 70.84\se{0.85} & \textbf{55.09}\se{1.12} & \textbf{78.18}\se{0.93} & \textbf{55.09}\se{1.12} & \textbf{68.64}\se{0.90} \\
& MASS-DPO & \textbf{57.05}\se{1.11} & 80.70\se{0.88} & \textbf{57.05}\se{1.11} & \textbf{71.08}\se{0.86} & 54.18\se{1.12} & 77.57\se{0.94} & 54.18\se{1.12} & 68.03\se{0.90} \\
\midrule
\multirow{5}{*}{Llama3}
& DPO      & 55.15\se{1.11} & 80.35\se{0.89} & 55.15\se{1.11} & 70.06\se{0.87} & 34.48\se{1.07} & 63.56\se{1.08} & 34.48\se{1.07} & 51.31\se{0.95} \\
& DMPO     & 49.95\se{1.12} & 78.35\se{0.92} & 49.95\se{1.12} & 66.76\se{0.88} & 27.82\se{1.01} & 58.72\se{1.11} & 27.82\se{1.01} & 45.69\se{0.94} \\
& DPO-k    & 56.05\se{1.11} & 80.30\se{0.89} & 56.05\se{1.11} & 70.41\se{0.87} & 43.95\se{1.11} & 70.77\se{1.02} & 43.95\se{1.11} & 59.69\se{0.94} \\
& S-DPO    & 56.50\se{1.11} & 80.85\se{0.88} & 56.50\se{1.11} & 70.95\se{0.86} & 48.94\se{1.12} & 73.39\se{0.99} & 48.94\se{1.12} & 63.25\se{0.94} \\
& MASS-DPO & \textbf{56.60}\se{1.11} & \textbf{81.15}\se{0.87} & \textbf{56.60}\se{1.11} & \textbf{71.17}\se{0.86} & \textbf{50.66}\se{1.12} & \textbf{76.01}\se{0.96} & \textbf{50.66}\se{1.12} & \textbf{65.57}\se{0.91} \\
\bottomrule
\end{tabular}
\label{tab:lastfm_movielens_allmethods}
\end{table*}

\subsection{How effectively does D-optimal active negative selection optimize the multi-negative preference learning objective?}
\label{sec:exp-rq1}

We compare MASS-DPO's active negative selection to the softmax-based random selection in S-DPO across all four datasets. Figure~\ref{fig:lastfm_medmcqa} and Figure~\ref{fig:qasc_movielens} (Appendix~\ref{app:exp}) track three alignment metrics during training: \emph{margin} (logit gap between preferred vs.\ rejected), \emph{accuracy}, and \emph{chosen rewards}. Across datasets, MASS-DPO (solid) achieves larger margins and faster early gains than S-DPO (dashed), with the gap emerging early and persisting through training. \emph{Accuracy} follows the same pattern: curves for MASS-DPO rise more quickly and attain higher plateaus. Finally, \emph{chosen-reward trajectories} under MASS-DPO are smoother and more stable across steps, while S-DPO exhibits noticeably noisier dynamics.

\begin{figure*}[t]
\centering
\includegraphics[width=0.85\linewidth]{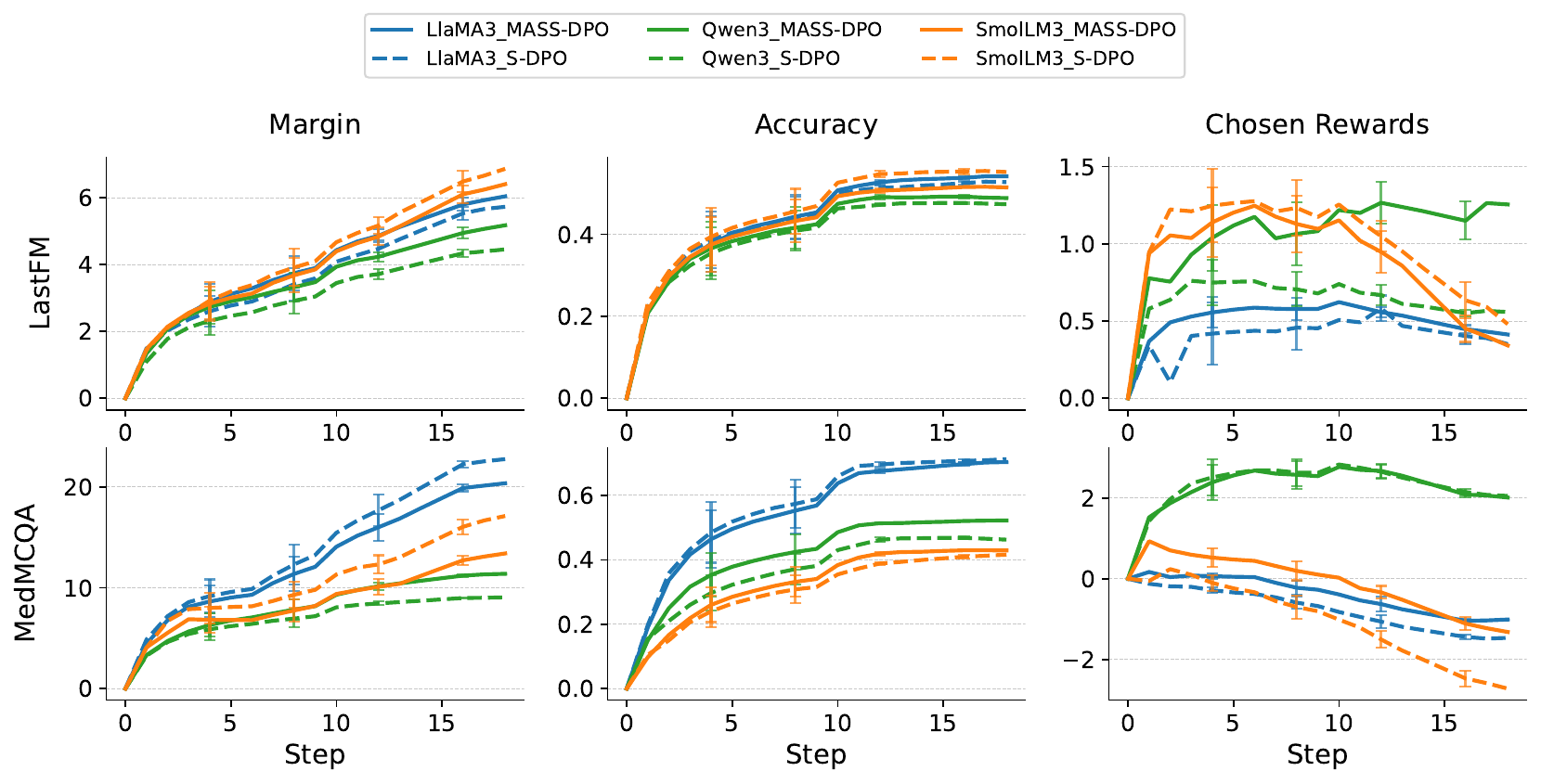}
\caption{Training dynamics on LastFM and MedMCQA. MASS-DPO (solid) achieves larger margins, higher accuracy, and more stable chosen rewards than S-DPO (dashed).}
\label{fig:lastfm_medmcqa}
\end{figure*}

\subsection{How does MASS-DPO improve downstream policy performance compared to existing preference optimization methods?}
\label{sec:exp-rq2}
We benchmark MASS-DPO against DPO, DMPO, DPO-k, and S-DPO on four datasets (MedMCQA, QASC, LastFM, MovieLens) using \emph{Accuracy}, reporting results for three base models in Table~\ref{table:results}. MASS-DPO achieves the highest average accuracy on Qwen3 and SmolLM3, leads both recommendation tasks on Llama3, and remains competitive on every dataset across all three model families.
Baselines without active negative selection (DPO, DMPO, and DPO-k) generally underperform, confirming that \emph{which} negatives enter the multi-negative objective matters. At matched wall-clock budgets (Table~\ref{tab:wallclock}, Appendix~\ref{app:compute}), MASS-DPO with a small selected subset matches or exceeds training with all available negatives (Table~\ref{tab:fullpool}), indicating that well-chosen negatives provide more useful signal per gradient step than the full pool.

\subsection{How informative are the negatives produced by incremental selection?}
\label{sec:exp-rq3}

We assess negative-selection quality using downstream utility metrics, MRR and Margin (Table~\ref{tab:mrr-margin-all}, Appendix~\ref{app:results}), and ranking quality on recommendation and QA via Recall/NDCG at $k\!\in\!\{1,3\}$ (main Table~\ref{table:results}; Appendix~\ref{app:results}, Tables~\ref{tab:lastfm_movielens_allmethods} and~\ref{tab:medmcqa_qasc_allmethods}). Across base models and datasets, MASS-DPO improves \emph{MRR} over S-DPO and delivers higher or comparable \emph{Margins}. On ranking metrics, MASS-DPO attains best or tied-best scores on most cells across all four metrics \{R@1, R@3, N@1, N@3\}, demonstrating stronger ranking quality on both recommendation and QA. These results confirm that active negative selection produces harder, more informative training pairs that translate into stronger ranking and alignment quality.

\begin{table*}
\small
\centering
\caption{MASS-DPO ablations on two hyperparameters. Each entry reports accuracy$_{\text{SE}}$, where the subscript denotes standard error. \textbf{Bold} = best. (a) Varying the scale $\beta\in\{0.1,0.5,1.0\}$ while holding $n=3$ fixed. (b) Varying $n\in\{1,3,5\}$ while holding $\beta=0.1$ fixed.}
\setlength{\tabcolsep}{3.5pt}
\renewcommand{\arraystretch}{1.05}

\begin{subtable}{.49\textwidth}
\centering
\caption{$\beta$ ablation}
\label{tab:beta_ablation}
\begin{adjustbox}{width=\linewidth}
\begin{tabular}{l|l|cccc}
\toprule
\multirow{1}{*}{Model} & \multirow{1}{*}{$\beta$} & Medmcqa & QASC & LastFM & MovieLens \\
\midrule
\multirow{3}{*}{Qwen3}
& 0.1 & \textbf{56.66\se{0.77}} & \textbf{72.19\se{1.05}} & \textbf{52.30\se{0.79}} & \textbf{47.58\se{0.80}} \\
& 0.5 & 46.29\se{0.79} & 71.41\se{1.06} & 48.15\se{0.81} & 39.82\se{0.79} \\
& 1.0 & 43.49\se{0.77} & 69.65\se{1.06} & 44.15\se{0.80} & 34.12\se{0.74} \\
\midrule
\multirow{3}{*}{SmolLM3}
& 0.1 & \textbf{44.19\se{0.79}} & \textbf{71.63\se{1.07}} & \textbf{57.25\se{0.79}} & 54.03\se{0.77} \\
& 0.5 & 39.73\se{0.76} & \textbf{71.63\se{1.03}} & 54.75\se{0.79} & \textbf{56.30\se{0.79}} \\
& 1.0 & 35.42\se{0.75} & 68.98\se{1.06} & 52.00\se{0.80} & 52.12\se{0.80} \\
\midrule
\multirow{3}{*}{Llama3}
& 0.1 & \textbf{71.29\se{0.74}} & \textbf{73.62\se{1.03}} & \textbf{57.35\se{0.81}} & 49.70\se{0.80} \\
& 0.5 & 69.69\se{0.74} & 73.51\se{1.01} & 55.75\se{0.80} & \textbf{51.06\se{0.78}} \\
& 1.0 & 66.28\se{0.76} & 72.19\se{1.02} & 52.25\se{0.81} & 45.92\se{0.78} \\
\bottomrule
\end{tabular}
\end{adjustbox}
\end{subtable}\hfill
\begin{subtable}{.49\textwidth}
\centering
\caption{Negatives $n$ ablation}
\label{tab:negatives_ablation}
\begin{adjustbox}{width=\linewidth}
\begin{tabular}{l|l|cccc}
\toprule
\multirow{1}{*}{Model} & \multirow{1}{*}{{$n$}} & Medmcqa & QASC & LastFM & MovieLens \\
\midrule
\multirow{3}{*}{Qwen3}
& 1 & 50.95\se{0.78} & 68.21\se{1.13} & 47.80\se{0.80} & 32.86\se{0.73} \\
& 3 & 56.66\se{0.77} & 72.19\se{1.05} & 52.30\se{0.79} & 47.58\se{0.80} \\
& 5 & \textbf{57.31\se{0.75}} & \textbf{73.73\se{1.03}} & \textbf{54.50\se{0.79}} & \textbf{58.11\se{0.78}} \\
\midrule
\multirow{3}{*}{SmolLM3}
& 1 & 29.26\se{0.72} & 65.67\se{1.09} & 50.70\se{0.81} & 34.58\se{0.75} \\
& 3 & 44.19\se{0.79} & \textbf{71.63\se{1.07}} & 57.25\se{0.79} & 54.03\se{0.77} \\
& 5 & \textbf{46.59\se{0.79}} & \textbf{71.63\se{1.04}} & \textbf{59.50\se{0.77}} & \textbf{65.07\se{0.74}} \\
\midrule
\multirow{3}{*}{Llama3}
& 1 & 46.99\se{0.78} & 71.96\se{1.03} & 52.70\se{0.81} & 32.71\se{0.73} \\
& 3 & 71.29\se{0.74} & 73.62\se{1.03} & 57.35\se{0.81} & 49.70\se{0.80} \\
& 5 & \textbf{73.55\se{0.71}} & \textbf{74.94\se{0.98}} & \textbf{60.05\se{0.80}} & \textbf{60.99\se{0.77}} \\
\bottomrule
\end{tabular}
\end{adjustbox}
\end{subtable}
\label{tab:ablations_side_by_side}
\end{table*}

\subsection{Ablation Studies}
\label{sec:ablations}

MASS-DPO's behavior is governed by the Fisher-information (\Cref{eq:fis}) and the D-optimal selection objective (\Cref{eq:d-opt}). We therefore ablate two key knobs predicted by theory to matter most: the preference-logit scale $\beta$ and the number of selected negatives $n$.

\paragraph{Effect of $\beta$.}
\label{sec:beta-ablation}
The scale $\beta$ is shared between the DPO training loss and the D-optimal selection objective, where it controls the sharpness of the softmax weights over candidates. Sweeping $\beta\!\in\!\{0.1,0.5,1.0\}$ across three model families (Table~\ref{tab:beta_ablation}), we find $\beta=0.1$ consistently yields the strongest results.

\paragraph{Number of negatives ($n$).}
\label{sec:k-ablation}
D-optimal design predicts that adding more negatives improves parameter estimation until coverage of the information space saturates. Varying the selected-negative budget $n\!\in\!\{1,3,5\}$ shows monotonic gains from $n{=}1\!\rightarrow\!3\!\rightarrow\!5$ across models and datasets  (Table~\ref{tab:negatives_ablation}). These results indicate the incremental sample selection procedure reliably assembles complementary negatives that expand $\log\det$ of the information matrix, aligning empirical improvements with our D-optimal design analysis.

\section{Conclusion}

We introduced MASS-DPO, a within-prompt active negative selection method for multi-negative preference optimization under the Plackett--Luce model. By deriving a PL-specific Fisher-information objective and formulating negative selection as a D-optimal design problem, MASS-DPO selects compact subsets that retain the full pool's information while reducing redundancy via an efficient incremental rank-one algorithm. Experiments across four benchmarks and three model families confirm that MASS-DPO delivers stronger alignment with substantially fewer negatives.

\bibliographystyle{plainnat}

\appendix

\section{Appendix}

\begin{lemma}[Gradient Derivation]
\label{app:gradient_derivation}
Consider the loss for a single sample
\begin{equation}
L(\theta) = -\log \sigma\Bigl(Z(\theta)\Bigr), \quad \text{with} \quad Z(\theta) = -\log \left( \sum_{j\in S_n} \exp\Bigl[\beta\left(\phi_j^\top \theta + b_j\right)\Bigr] \right),
\end{equation}

\[
\frac{d}{dz}\bigl[-\log \sigma(Z(\theta))\bigr] = -\frac{1}{\sigma(Z(\theta))} \cdot \sigma'(Z(\theta))
= -\frac{\sigma(Z(\theta))(1-\sigma(Z(\theta)))}{\sigma(Z(\theta))} = -(1-\sigma(Z(\theta))).
\]
\begin{equation}
\frac{\partial L}{\partial Z(\theta)} = -(1-\sigma(Z(\theta))).
\end{equation}

\[
A(\theta) = \sum_{j\in S_n} \exp\Bigl[\beta\left(\phi_j^\top \theta + b_j\right)\Bigr],
\]
so that \(Z(\theta) = -\log A(\theta)\). Then,
\[
\frac{\partial Z(\theta)}{\partial \theta} = -\frac{1}{A(\theta)} \frac{\partial A(\theta)}{\partial \theta},
\]
\[
\frac{\partial A(\theta)}{\partial \theta} = \sum_{j\in S_n}  \exp\Bigl[\beta\left(\phi_j^\top \theta + b_j\right)\Bigr] \beta\phi_j,
\]
\[
\frac{\partial Z(\theta)}{\partial \theta} = -\beta \sum_{j\in S_n} \frac{\exp\Bigl[\beta\left(\phi_j^\top \theta + b_j\right)\Bigr]}{A(\theta)} \phi_j 
= -\beta \sum_{j\in S_n} q_j^{S_n}(\theta) \phi_j,
\]
where the softmax weights are defined as
\[
q_j^{S_n}(\theta) = \frac{\exp\Bigl[\beta\left(\phi_j^\top \theta + b_j\right)\Bigr]}{A(\theta)}.
\]

\[
\frac{\partial L}{\partial \theta} = \frac{\partial L}{\partial Z(\theta)}\cdot\frac{\partial Z(\theta)}{\partial \theta}
= -(1-\sigma(Z(\theta))) \cdot \Bigl[-\beta \sum_{j\in S_n} q_j^{S_n}(\theta) \phi_j\Bigr]
= \beta (1-\sigma(Z(\theta))) \sum_{j\in S_n} q_j^{S_n}(\theta) \phi_j.
\]
Thus, the gradient of the loss is
\begin{equation}
\nabla_\theta L = \beta (1-\sigma(Z(\theta))) \sum_{j\in S_n} q_j^{S_n}(\theta) \phi_j.
\end{equation}
Equivalently, defining $\bar\phi_{S_n}(\theta)=\sum_{j\in S_n} q_j^{S_n}(\theta)\phi_j$, we have
$\nabla_\theta L=\beta(1-\sigma(Z(\theta)))\bar\phi_{S_n}(\theta)$.
\end{lemma}

\begin{lemma}[Hessian Derivation]
\label{app:hessian_derivation}

Recall the multi-negative DPO loss:
\[
L(\theta; S_n) = -\log\sigma\left(-\log\sum_{i \in S_n}\exp(\beta(\phi_i^\top\theta + b_i))\right),
\]
where $\sigma(\cdot)$ denotes the sigmoid function. Throughout this proof we abbreviate the subset-normalized weights and subset mean from~\Cref{lemma:hessian} as $q_j:=q_j^{S_n}(\theta)$ and $\bar\phi:=\bar\phi_{S_n}(\theta)$, and write
\[
Z_n = -\log\sum_{i \in S_n}\exp(\beta(\phi_i^\top\theta + b_i)), \quad q_j = \frac{\exp(\beta(\phi_j^\top\theta + b_j))}{\sum_{k \in S_n}\exp(\beta(\phi_k^\top\theta + b_k))}, \quad \bar\phi = \sum_{j \in S_n} q_j\phi_j.
\]

Starting from the gradient~\Cref{eq:grad},
\[
\nabla_\theta L(\theta; S_n) = \beta(1 - \sigma(Z_n)) \bar\phi,
\]
we derive the Hessian by differentiating again with respect to $\theta$:
\begin{align}
\nabla_\theta^2 L(\theta; S_n)
&= \beta\nabla_\theta\left[(1 - \sigma(Z_n))\bar\phi\right] \\
&= \beta(1-\sigma(Z_n))\nabla_\theta\bar\phi
 - \beta\sigma(Z_n)(1-\sigma(Z_n))\bar\phi\,\nabla_\theta Z_n^\top.
\end{align}

Expanding the first term using the definition of $q_j$ gives:
\begin{align}
\nabla_\theta\bar\phi
&= \beta\sum_{j\in S_n}q_j\phi_j\phi_j^\top - \beta\left(\sum_{j\in S_n}q_j\phi_j\right)\left(\sum_{j\in S_n}q_j\phi_j\right)^\top\\
&= \beta\sum_{j\in S_n}q_j(\phi_j-\bar\phi)(\phi_j-\bar\phi)^\top.
\end{align}

Note also that:
\[
\nabla_\theta Z_n = -\beta\sum_{j\in S_n}q_j\phi_j = -\beta\bar\phi.
\]

Thus, substituting back, the Hessian becomes:
\begin{align}
\nabla_\theta^2 L(\theta; S_n) 
&= \beta^2(1 - \sigma(Z_n))\sum_{j\in S_n}q_j(\phi_j-\bar\phi)(\phi_j-\bar\phi)^\top + \beta^2\sigma(Z_n)(1-\sigma(Z_n))\bar\phi\bar\phi^\top \\
&= \beta^2(1 - \sigma(Z_n))\left[\sigma(Z_n)\,\bar\phi\bar\phi^\top + \sum_{j\in S_n}q_j(\phi_j-\bar\phi)(\phi_j-\bar\phi)^\top\right].
\end{align}
\end{lemma}

\begin{remark}[Loewner lower bound in \Cref{eq:hessian_lb}]
\label{rem:eq10_justification}

From the decomposition above,
\[
\nabla_\theta^2 L(\theta; S_n)
=
\beta^2(1-\sigma(Z_n))
\Bigl[
\sigma(Z_n)\,\bar\phi\bar\phi^\top
+
\sum_{j\in S_n}q_j(\phi_j-\bar\phi)(\phi_j-\bar\phi)^\top
\Bigr].
\]
The rank one term $\sigma(Z_n)\,\bar\phi\bar\phi^\top$ is positive semidefinite, and the weighted covariance term is also positive semidefinite. Therefore dropping the rank one term yields the Loewner lower bound
\[
\nabla_\theta^2 L(\theta; S_n)
\succeq
\beta^2(1-\sigma(Z_n))
\sum_{j\in S_n}q_j(\phi_j-\bar\phi)(\phi_j-\bar\phi)^\top,
\]
which is~\Cref{eq:hessian_lb}.
\end{remark}

\begin{thm}[Selected-Subset Estimator Stability]
\label{app:1-parameter-estimation}
Let
\[
F_{\mathcal C}(\theta)=L(\theta;\mathcal C)+\frac{\gamma}{2}\|\theta\|_2^2,
\qquad
F_{S_n}(\theta)=L(\theta;S_n)+\frac{\gamma}{2}\|\theta\|_2^2,
\]
and let $\theta_*=\argmin_{\theta\in\mathbb R^d}F_{\mathcal C}(\theta)$ and
$\hat\theta_n=\argmin_{\theta\in\mathbb R^d}F_{S_n}(\theta)$. Define the integrated subset curvature
\[
\bar\Sigma_n=\int_0^1 \nabla^2 F_{S_n}\bigl(\theta_*+t(\hat\theta_n-\theta_*)\bigr)\,dt .
\]
If $\bar\Sigma_n\succeq\Sigma_0\succ0$, then
\[
\|\hat\theta_n-\theta_*\|_{\Sigma_0}
\leq
\|\nabla F_{S_n}(\theta_*)\|_{\Sigma_0^{-1}}
=
\|\nabla L(\theta_*;S_n)-\nabla L(\theta_*;\mathcal C)\|_{\Sigma_0^{-1}} .
\]
Thus an estimator-stability event of the form used in~\Cref{thm:logit-error} is guaranteed whenever the selected-subset gradient discrepancy on the right is at most $\eta_n$ in the corresponding dual norm.
\end{thm}

\begin{proof}
The first-order condition for the selected regularized objective gives $\nabla F_{S_n}(\hat\theta_n)=0$. By the fundamental theorem of calculus,
\[
0=\nabla F_{S_n}(\theta_*)+\bar\Sigma_n(\hat\theta_n-\theta_*).
\]
Therefore $\hat\theta_n-\theta_*=-\bar\Sigma_n^{-1}\nabla F_{S_n}(\theta_*)$. Since $\bar\Sigma_n\succeq\Sigma_0$, the stated norm bound follows. Finally, $\nabla F_{\mathcal C}(\theta_*)=0$, so $\nabla F_{S_n}(\theta_*)=\nabla L(\theta_*;S_n)-\nabla L(\theta_*;\mathcal C)$.
\end{proof}

\begin{thm}[Batch Estimator Stability]
\label{app:k-parameter-estimation}
For an analysis block of $k$ prompts, define the averaged full-pool and selected-subset losses
\[
L_k(\theta;\mathcal C_{1:k})=\frac{1}{k}\sum_{i=1}^k L_i(\theta;\mathcal C_i),
\qquad
L_k(\theta;S_{1:k,n})=\frac{1}{k}\sum_{i=1}^k L_i(\theta;S_{i,n}).
\]
Let $F_{\mathcal C,k}(\theta)=L_k(\theta;\mathcal C_{1:k})+\frac{\gamma}{2}\|\theta\|_2^2$ and
$F_{S,k}(\theta)=L_k(\theta;S_{1:k,n})+\frac{\gamma}{2}\|\theta\|_2^2$, with minimizers
$\theta_{*,k}$ and $\hat\theta_{k,n}$ respectively. If the integrated Hessian of $F_{S,k}$ along the segment from $\theta_{*,k}$ to $\hat\theta_{k,n}$ lower-bounds $\Sigma_{0,k}\succ0$, then
\[
\|\hat\theta_{k,n}-\theta_{*,k}\|_{\Sigma_{0,k}}
\leq
\|\nabla L_k(\theta_{*,k};S_{1:k,n})-\nabla L_k(\theta_{*,k};\mathcal C_{1:k})\|_{\Sigma_{0,k}^{-1}}.
\]
Moreover,
\[
\nabla L_k(\theta;S_{1:k,n})
=
\frac{\beta}{k}\sum_{i=1}^k(1-\sigma(Z_i(\theta)))
\sum_{j\in S_{i,n}}q_{i,j}^{S_{i,n}}(\theta)\phi_{i,j}.
\]
\end{thm}

\begin{proof}
The proof is identical to the single-prompt perturbation argument in~\Cref{app:1-parameter-estimation}, with $F_{\mathcal C}$ and $F_{S_n}$ replaced by the averaged objectives $F_{\mathcal C,k}$ and $F_{S,k}$. The displayed gradient follows by differentiating the average loss term by term, which introduces the factor $1/k$.
\end{proof}

\section{Formal Relative Logit Error Bound}
\label{app:thm-formal}

\begin{thm}[Relative Logit Error Bound --- Formal Version of Theorem~\ref{thm:logit-error}]\ \\
\label{thm:logit-error-formal}
For a fixed prompt, let $\mathcal C=\{y_i\}_{i=1}^N$ be its candidate-negative pool. Let
\[
\theta_* = \argmin_{\theta\in\mathbb R^d}\Bigl[L(\theta;\mathcal C)+\frac{\gamma}{2}\|\theta\|_2^2\Bigr],
\qquad
\hat{\theta}_n = \argmin_{\theta\in\mathbb R^d}\Bigl[L(\theta;S_n)+\frac{\gamma}{2}\|\theta\|_2^2\Bigr],
\]
where $S_n\subseteq\mathcal C$ is the selected subset.

Define the fixed full-pool weights used by Algorithm~\ref{alg:dopt},
\[
q_i^0 =
\frac{\exp(\beta(\phi_i^\top\theta_0+b_i))}
{\sum_{\ell\in\mathcal C}\exp(\beta(\phi_\ell^\top\theta_0+b_\ell))},
\qquad
\bar\phi_0=\sum_{i\in\mathcal C}q_i^0\phi_i,
\]
with $Z_{\mathcal C}^0=-\log\sum_{\ell\in\mathcal C}\exp[\beta(\phi_\ell^\top\theta_0+b_\ell)]$. Set $\tilde\phi_i^0=\phi_i-\bar\phi_0$, $v_i^0=\sqrt{q_i^0}\tilde\phi_i^0$,
\[
I_n=\{i:y_i\in S_n\},\qquad
H_n^0=\gamma I+\alpha_0\sum_{i\in I_n}v_i^0(v_i^0)^\top,\qquad
\alpha_0=\beta^2(1-\sigma(Z_\mathcal C^0)).
\]
Let $q_{\min}^0=\min_{i\in\mathcal C}q_i^0$, $L_v^0=\max_{i\in\mathcal C}\|v_i^0\|_2^2$, and define
\[
B_n =
\frac{\kappa}{\rho q_{\min}^0}
\cdot
\frac{1+\alpha_0 L_v^0/\gamma}{\alpha_0}
\cdot
\frac{d}{n}
\log\!\left(1+\frac{\alpha_0 nL_v^0}{\gamma d}\right).
\]
Under Assumptions~\ref{assump:boundedness}--\ref{assump:fisher-compat}, with probability at least $1-\delta$ over the stochastic preference observations drawn from the PL model,
\[
\mathcal{E}_{\mathrm{rel}}(\hat\theta_n,\theta_*)
\;\leq\;
\widetilde O\!\left(d\sqrt{\frac{\log(1/\delta)}{n}}\right),
\]
where $\widetilde O$ hides logarithmic factors and the candidate-pool regularity constants $\rho, q_{\min}^0, \kappa, \alpha_0, \gamma, L_v^0$.
\end{thm}

\begin{proof}
Treating the multi-negative loss as a generalized linear model and applying self-normalized concentration to the regularized estimator~\citep{abbasi2011improved,kveton2025active}, with probability at least $1-\delta$,
\[
\|\hat\theta_n-\theta_*\|_{\Sigma_n}\;\leq\;C\sqrt{d\log(1/\delta)},
\]
for a constant $C$ depending on $\beta,c_{\min},c_{\max},\gamma$ from Assumptions~\ref{assump:boundedness}--\ref{assump:design-matrix}.
For any $i,j\in\mathcal C$, since $\phi_i-\phi_j=\tilde\phi_i^0-\tilde\phi_j^0$, Cauchy--Schwarz gives
\[
|(\phi_i-\phi_j)^\top(\hat\theta_n - \theta_*)|
\leq
\bigl(\|\tilde\phi_i^0\|_{\Sigma_n^{-1}}+\|\tilde\phi_j^0\|_{\Sigma_n^{-1}}\bigr)
\|\hat\theta_n - \theta_*\|_{\Sigma_n}.
\]
By~\Cref{thm:leverage} under Assumption~\ref{assump:fisher-compat}, $(\tilde\phi_i^0)^\top\Sigma_n^{-1}\tilde\phi_i^0\le B_n$, so each leverage term is at most $\sqrt{B_n}$. Combining,
\[
\mathcal E_{\mathrm{rel}}\;\leq\;2\sqrt{B_n}\,\|\hat\theta_n-\theta_*\|_{\Sigma_n}\;\leq\;\widetilde O\!\left(d\sqrt{\log(1/\delta)/n}\right),
\]
where the candidate-pool constants are absorbed into $\widetilde O$.
\end{proof}

\section{Technical Assumptions}
\label{app:assumptions}

\subsection{Assumption Details}
\begin{assumption}[Bounded Feature Differences and Bias]
\label{assump:boundedness}
For each prompt and candidate set, the feature-difference vectors used in the loss,
$\phi_i=\phi(x,y_i)-\phi(x,y^*)$, and the reference-policy offsets are bounded:
\[
\|\phi_i\|_2 \leq L_\phi, \quad |b_i| \leq L_b.
\]
For the theoretical analysis, the ridge-regularized objectives are optimized over $\mathbb R^d$; we assume the relevant minimizers lie in a ball $\|\theta_*\|_2,\|\hat\theta_n\|_2\le R_\theta$. The bounds hide polynomial dependence on the finite constants $L_\phi$ and $L_b$.
\end{assumption}

\begin{assumption}[Bounded Curvature Scale on the Relevant Region]
\label{assump:design-matrix}
Let $Z_\mathcal C(\theta)=-\log\sum_{j\in\mathcal C}\exp[\beta(\phi_j^\top\theta+b_j)]$. On the compact parameter region containing $\theta_0$, $\theta_*$, $\hat\theta_n$, and the line segments used in the perturbation arguments of Appendix~A, there exist constants $0 < c_{\min} \leq c_{\max} \leq \beta^2$ such that, for all $\theta$ in this region,
\[
c_{\min} \leq \beta^2 (1 - \sigma(Z_\mathcal C(\theta))) \leq c_{\max}.
\]
\end{assumption}

\begin{assumption}[Diverse Candidate Set]
\label{assump:diverse}
Let $q_j^0$ and $\bar\phi_0$ be the full-pool softmax weights and center computed by Algorithm~\ref{alg:dopt} at preprocessing parameter $\theta_0$, and write $Z_{\mathcal C}^0=Z_\mathcal C(\theta_0)$. Let $v_j^0=\sqrt{q_j^0}(\phi_j-\bar\phi_0)$ be the fixed full-pool-centered Fisher contribution used by Algorithm~\ref{alg:dopt}, and let $\alpha_0=\beta^2(1-\sigma(Z_{\mathcal C}^0))$. There exists a constant $\kappa \geq 1$ such that for any selected index set $I_k$ produced by Algorithm~\ref{alg:dopt} after $k$ steps ($k=0,1,\ldots,n-1$, with $I_0=\emptyset$), with design matrix $H_k^0 = \gamma I + \alpha_0\sum_{j\in I_k} v_j^0(v_j^0)^\top$, we have
\[
(v_i^0)^\top (H_k^0)^{-1} v_i^0 \;\leq\; \kappa \cdot \max_{j \in [N] \setminus I_k}\; (v_j^0)^\top (H_k^0)^{-1} v_j^0, \quad \forall\, i \in [N],\; \forall\, k \in \{0,1,\ldots,n-1\}.
\]
\end{assumption}

Here $\kappa$ measures how well the remaining pool continues to cover high-leverage directions after each greedy step; $\kappa=1$ in the variant allowing reselection.

\begin{assumption}[Fisher-Compatibility]
\label{assump:fisher-compat}
There exists $\rho\in(0,1]$ such that the regularized subset Hessian dominates the full-pool selection objective up to a constant:
\[
\Sigma_n := \gamma I + \nabla^2 L(\theta_*; S_n) \;\succeq\; \rho\, H_n^0,
\]
where $H_n^0=\gamma I+\alpha_0\sum_{i\in I_n}v_i^0(v_i^0)^\top$ is the selection objective of~\Cref{eq:fis} evaluated on the selected index set.
\end{assumption}

\subsection{Centered Leverage Score Bound}
\label{app:leverage}

The following result shows that the centered leverage scores decay at rate $\tilde O(d/(\rho q_{\min}^0 n))$ under Assumptions~\ref{assump:boundedness}--\ref{assump:fisher-compat}. This is the key ingredient linking the $\Sigma_n$-norm estimation error to the relative logit error in~\Cref{thm:logit-error}.

\begin{thm}[Centered Leverage Score Decay]
\label{thm:leverage}
Let $\tilde\phi_i^0=\phi_i-\bar\phi_0$, $v_i^0=\sqrt{q_i^0}\tilde\phi_i^0$, and $H_k^0=\gamma I+\alpha_0\sum_{t=1}^k v_{i_t}^0(v_{i_t}^0)^\top$, where $\alpha_0=\beta^2(1-\sigma(Z_{\mathcal C}^0))$ and $Z_{\mathcal C}^0=Z_\mathcal C(\theta_0)$. Let $q_{\min}^0=\min_{i\in\mathcal C}q_i^0$ and $L_v^0=\max_{i\in\mathcal C}\|v_i^0\|_2^2$. Under Assumptions~\ref{assump:boundedness}--\ref{assump:fisher-compat}, for any candidate $i \in \mathcal C$ and any subset $S_n$ of size $n$ produced by Algorithm~\ref{alg:dopt},
\[
(\phi_i-\bar\phi_0)^\top \Sigma_n^{-1}(\phi_i-\bar\phi_0)
\;\leq\;
\frac{\kappa}{\rho q_{\min}^0}\cdot
\frac{1+\alpha_0 L_v^0/\gamma}{\alpha_0}\cdot
\frac{d}{n}\log\!\Bigl(1+\frac{\alpha_0 n L_v^0}{\gamma d}\Bigr).
\]
\end{thm}

\begin{proof}
Let
\[
x_k = (v_{i_k}^0)^\top (H_{k-1}^0)^{-1}v_{i_k}^0.
\]
By the max-marginal construction in Algorithm~\ref{alg:dopt}, $x_k=\max_{j\notin I_{k-1}}(v_j^0)^\top (H_{k-1}^0)^{-1}v_j^0$. By Assumption~\ref{assump:diverse}, for any $i\in\mathcal C$,
\[
(v_i^0)^\top (H_{k-1}^0)^{-1}v_i^0\leq \kappa x_k.
\]
Since $H_n^0\succeq H_{k-1}^0$, we also have $(H_n^0)^{-1}\preceq (H_{k-1}^0)^{-1}$, and therefore
\[
(v_i^0)^\top (H_n^0)^{-1}v_i^0\leq \kappa x_k,\qquad \forall k\in[n].
\]
Thus $(v_i^0)^\top (H_n^0)^{-1}v_i^0\leq \kappa\min_k x_k\leq \frac{\kappa}{n}\sum_{k=1}^n x_k$.

It remains to bound $\sum_k x_k$. The matrix determinant lemma gives
\[
\log\det H_k^0-\log\det H_{k-1}^0
= \log(1+\alpha_0 x_k).
\]
Because $H_{k-1}^0\succeq \gamma I$ and $\|v_i^0\|_2^2\leq L_v^0$, we have $x_k\leq L_v^0/\gamma$. Hence
\[
\log(1+\alpha_0 x_k)
\geq
\frac{\alpha_0 x_k}{1+\alpha_0 L_v^0/\gamma},
\]
and
\[
\sum_{k=1}^n x_k
\leq
\frac{1+\alpha_0 L_v^0/\gamma}{\alpha_0}\,
\bigl(\log\det H_n^0-\log\det H_0^0\bigr).
\]
Using the determinant upper bound under a fixed trace,
\[
\log\det H_n^0-\log\det H_0^0
\leq
d\log\!\Bigl(1+\frac{\alpha_0\sum_{k=1}^n\|v_{i_k}^0\|_2^2}{\gamma d}\Bigr)
\leq
d\log\!\Bigl(1+\frac{\alpha_0 n L_v^0}{\gamma d}\Bigr).
\]
Combining these inequalities yields
\[
(v_i^0)^\top (H_n^0)^{-1}v_i^0
\leq
\kappa\frac{1+\alpha_0 L_v^0/\gamma}{\alpha_0}\,
\frac{d}{n}\log\!\Bigl(1+\frac{\alpha_0 n L_v^0}{\gamma d}\Bigr).
\]
Finally, since $\tilde\phi_i^0=v_i^0/\sqrt{q_i^0}$ and $\Sigma_n\succeq \rho H_n^0$ by Assumption~\ref{assump:fisher-compat},
\[
(\tilde\phi_i^0)^\top\Sigma_n^{-1}\tilde\phi_i^0
\leq
\frac{1}{\rho q_i^0}(v_i^0)^\top (H_n^0)^{-1}v_i^0
\leq
\frac{1}{\rho q_{\min}^0}(v_i^0)^\top (H_n^0)^{-1}v_i^0,
\]
which proves the claimed bound. When $q_{\min}^0$ is treated as a fixed candidate-pool regularity constant, the displayed expression simplifies to $\tilde O(d/(\rho n))$.
\end{proof}

\subsection{Experimental Settings} \label{app:exp}

To further manage computational costs, we cap the number of response candidates at 20 for the LastFM, MovieLens, and MedMCQA datasets, and at 8 for QASC. Although MedMCQA natively provides only four options per question, we expand this to 20 by pooling all candidates with the same \texttt{subject\_name} field. We also subsample each dataset to 20k training samples, 200 samples for online evaluation, and 2,000 samples for testing. All prompts are formatted using each model’s provided chat template to ensure consistent input structure across tasks.

\begin{figure*}[htp]
\centering
\includegraphics[width=0.85\linewidth]{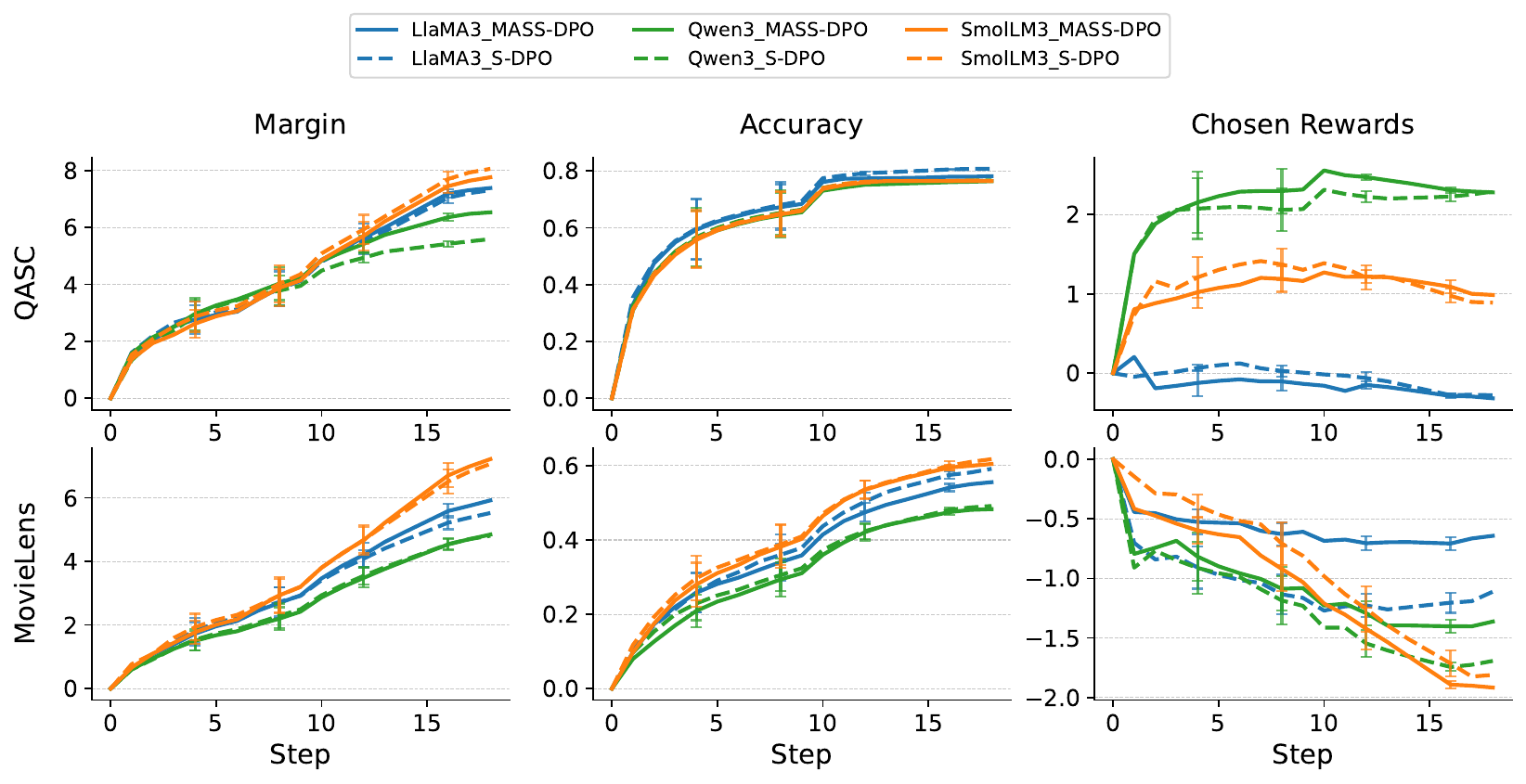}
\caption{Margin, accuracy, and chosen reward comparisons on MovieLens and QASC datasets. MASS-DPO achieves higher margins, superior accuracy, and more stable chosen rewards than S-DPO. The x-axis (Step) counts evaluations during training.}
\label{fig:qasc_movielens}
\end{figure*}

\paragraph{Scope.}\label{app:scope} Our evaluation is scoped to settings with a finite per-prompt candidate pool: recommendation (LastFM, MovieLens) and multiple-choice QA (MedMCQA, QASC), where negatives are well-defined and bounded in number. Extending MASS-DPO to open-ended generation or instruction tuning would require an upstream pool-construction step (e.g., sampling negatives from the reference policy), which we do not evaluate here.

\subsection{Implementation Details}
\label{app:imp}

We implement our experiments using PyTorch, leveraging three widely used pre-trained LLMs: \textit{Llama-3.2-3B-Instruct}~\citep{grattafiori2024llama}, \textit{SmolLM3}~\citep{bakouch2025smollm3}, and \textit{Qwen3-4B}~\citep{qwen3}. Each model undergoes full fine-tuning on 8 NVIDIA A100 GPUs with a per-device batch size of 2, gradient accumulation steps of 8, learning rate of $10^{-5}$, a cosine learning-rate scheduler with warmup ratio $0.05$, and the Paged AdamW optimizer for 3 epochs with a fixed DPO scale $\beta = 0.1$ across all main experiments.
We enable gradient checkpointing, gradient clipping is applied with a maximum norm of $0.3$, and evaluation uses a batch size of 2. We extract representation vectors by mean-pooling the final hidden states, using either (a) all tokens from the concatenated prompt–response sequence or (b) only the response tokens, where prompt positions are masked out. Both strategies use the same pretrained LLM and tokenization pipeline.
We compute the negative subset $S_n$ during dataset preprocessing, using the frozen preprocessing checkpoint $\theta_0$ to obtain embeddings and log-probabilities; the selected subsets remain fixed throughout training. The preprocessing score is
\[
s_i=\beta\bigl(\phi_i^\top\theta_0+b_i\bigr),
\qquad
b_i=\log\pi_{\mathrm{ref}}(y^*\mid x)-\log\pi_{\mathrm{ref}}(y_i\mid x),
\]
where $\beta$ is the same value used in the DPO/PL training loss and $\theta_0$ is the pretrained initialization before preference fine-tuning.

\paragraph{Stability of fixed selection across training.}
To assess whether the pre-selected negatives remain informative as training progresses, we recompute the selection at initialization, mid-training (epoch~1), and the final checkpoint on a 500-sample subset of MovieLens with Llama-3.2-3B-Instruct for selected-negative budgets $n\in\{5,10\}$. Across both budgets, the subset obtained at initialization and the one recomputed at the final checkpoint share at least 98.8\% exact overlap and 0.997 mean Jaccard similarity (\Cref{tab:selection-stability}); the self log-det values vary by at most 1\% across checkpoints. These results indicate that the Fisher geometry is stable over the training horizon and dynamic re-selection would recover essentially the same subset.

\begin{table}[h]
\centering
\caption{Stability of fixed MASS-DPO selection across training checkpoints on MovieLens with Llama-3.2-3B-Instruct. Metrics compare the initialization-time subset against the subset recomputed at the final checkpoint.}
\label{tab:selection-stability}
\small
\begin{tabular}{lccccc}
\toprule
$n$ & Exact match & Mean Jaccard & Top-1 match & Top-3 exact match & Obj. retention \\
\midrule
5 & 99.2\% & 0.9969 & 100\% & 99.8\% & 0.9994 \\
10 & 98.8\% & 0.9970 & 100\% & 99.6\% & 0.9990 \\
\bottomrule
\end{tabular}
\end{table}

For the main comparison, we fix $n{=}3$ negatives for all multi-negative methods, DPO-k, DMPO, S-DPO, and MASS-DPO, so these methods use the same number of negative responses per prompt. DPO uses a single negative by construction. MASS-DPO incurs only the additional selection overhead of~\Cref{alg:dopt}, which is amortized as a one-time preprocessing cost.

\paragraph{Hyperparameters.}
For the D-optimal selection objective (\Cref{eq:fis}), we set the ridge $\gamma=0.1$ for all runs to ensure $H(S)$ is well conditioned, and we use the same $\beta$ as in training. In the main experiments $\beta=0.1$, and we only vary $\beta$ in the ablation.

\subsection{Results}
\label{app:results}

\begin{table*}[ht]
\centering
\scriptsize
\caption{Recall (R) and NDCG (N) at k=\{1,3\} on MedMCQA and QASC. Each entry reports metric$_{\text{SE}}$, where the subscript denotes standard error.}
\setlength{\tabcolsep}{5pt}
\begin{tabular}{l l | cccc | cccc}
\toprule
\multirow{2}{*}{Model} & \multirow{2}{*}{Method} & \multicolumn{4}{c|}{MedMCQA} & \multicolumn{4}{c}{QASC} \\
& & R@1 & R@3 & N@1 & N@3 & R@1 & R@3 & N@1 & N@3 \\
\midrule
\multirow{5}{*}{Qwen3}
& DPO      & 39.25\se{1.09} & 84.51\se{0.81} & 39.25\se{1.09} & 65.37\se{0.77} & 67.77\se{1.55} & 90.62\se{0.97} & 67.77\se{1.55} & 81.29\se{1.04} \\
& DMPO     & 26.02\se{0.98} & 74.89\se{0.97} & 26.02\se{0.98} & 53.60\se{0.81} & 68.21\se{1.55} & 90.51\se{0.97} & 68.21\se{1.55} & 81.40\se{1.04} \\
& DPO-k    & 54.59\se{1.11} & 89.67\se{0.68} & 54.59\se{1.11} & 74.86\se{0.72} & 71.08\se{1.51} & \textbf{92.72}\se{0.86} & 71.08\se{1.51} & \textbf{83.95}\se{0.96} \\
& S-DPO    & 51.03\se{1.12} & 86.52\se{0.76} & 51.03\se{1.12} & 71.54\se{0.77} & 70.42\se{1.52} & 91.50\se{0.93} & 70.42\se{1.52} & 83.04\se{1.00} \\
& MASS-DPO & \textbf{56.34}\se{1.11} & \textbf{89.72}\se{0.68} & \textbf{56.34}\se{1.11} & \textbf{75.62}\se{0.72} & \textbf{71.85}\se{1.49} & 91.61\se{0.92} & \textbf{71.85}\se{1.49} & 83.73\se{0.99} \\
\midrule
\multirow{5}{*}{SmolLM3}
& DPO      & 33.33\se{1.06} & 81.75\se{0.86} & 33.33\se{1.06} & 61.01\se{0.78} & 67.11\se{1.56} & 90.07\se{0.99} & 67.11\se{1.56} & 80.70\se{1.06} \\
& DMPO     & 26.02\se{0.98} & 75.39\se{0.96} & 26.02\se{0.98} & 53.93\se{0.80} & 66.11\se{1.57} & 88.74\se{1.05} & 66.11\se{1.57} & 79.54\se{1.10} \\
& DPO-k    & 44.16\se{1.11} & 85.91\se{0.78} & 44.16\se{1.11} & 68.09\se{0.77} & 70.31\se{1.52} & 90.18\se{0.99} & 70.31\se{1.52} & 82.06\se{1.05} \\
& S-DPO    & \textbf{45.46}\se{1.11} & 87.22\se{0.75} & \textbf{45.46}\se{1.11} & \textbf{69.54}\se{0.75} & 70.53\se{1.51} & 91.39\se{0.93} & 70.53\se{1.51} & 82.80\se{1.01} \\
& MASS-DPO & 44.81\se{1.11} & \textbf{87.47}\se{0.74} & 44.81\se{1.11} & 69.40\se{0.74} & \textbf{72.52}\se{1.48} & \textbf{91.83}\se{0.91} & \textbf{72.52}\se{1.48} & \textbf{83.82}\se{0.99} \\
\midrule
\multirow{5}{*}{Llama3}
& DPO      & 51.48\se{1.12} & 87.82\se{0.73} & 51.48\se{1.12} & 72.30\se{0.75} & 71.30\se{1.50} & 91.72\se{0.92} & 71.30\se{1.50} & 83.41\se{1.00} \\
& DMPO     & 24.36\se{0.96} & 74.99\se{0.97} & 24.36\se{0.96} & 52.84\se{0.80} & 69.32\se{1.53} & 91.94\se{0.90} & 69.32\se{1.53} & 82.67\se{0.99} \\
& DPO-k    & 71.13\se{1.01} & 93.88\se{0.54} & 71.13\se{1.01} & 84.51\se{0.62} & 73.95\se{1.46} & 92.38\se{0.88} & 73.95\se{1.46} & 84.93\se{0.96} \\
& S-DPO    & \textbf{72.33}\se{1.00} & 94.34\se{0.52} & \textbf{72.33}\se{1.00} & \textbf{85.20}\se{0.61} & \textbf{74.17}\se{1.45} & 92.38\se{0.88} & \textbf{74.17}\se{1.45} & \textbf{85.13}\se{0.96} \\
& MASS-DPO & 71.23\se{1.01} & \textbf{94.49}\se{0.51} & 71.23\se{1.01} & 84.84\se{0.61} & 73.84\se{1.46} & \textbf{92.60}\se{0.87} & 73.84\se{1.46} & \textbf{85.13}\se{0.95} \\
\bottomrule
\end{tabular}
\label{tab:medmcqa_qasc_allmethods}
\end{table*}

\begin{table*}[htp]
\centering
\caption{MRR and Margin across four datasets. Each cell shows \emph{MRR / Margin}.}
\small
\begin{tabular}{l l c c c c | c}
\toprule
\multirow{1}{*}{Model} &
\multirow{1}{*}{Method} &
\multicolumn{1}{c}{MedMCQA} &
\multicolumn{1}{c}{QASC} &
\multicolumn{1}{c}{LastFM} &
\multicolumn{1}{c}{MovieLens} &
\multicolumn{1}{|c}{{Average$\uparrow$}} \\
& & \multicolumn{1}{c}{MRR/Margin} & \multicolumn{1}{c}{MRR/Margin} &
\multicolumn{1}{c}{MRR/Margin} & \multicolumn{1}{c}{MRR/Margin} &
\multicolumn{1}{|c}{MRR/Margin} \\
\midrule

\multirow{2}{*}{Qwen3}
& S\mbox{-}DPO
& 69.74 / 9.33
& 81.87 / 5.29
& 64.12 / 4.67
& {61.29} / \textbf{4.98}
& {69.26} / 6.07 \\
& MASS\mbox{-}DPO
& \textbf{73.30} / \textbf{11.76}
& \textbf{82.71} / \textbf{6.22}
& \textbf{66.28} / \textbf{5.51}
& \textbf{61.61} / 4.94
& \textbf{70.97} / \textbf{7.11} \\
\midrule

\multirow{2}{*}{SmolLM3}
& S\mbox{-}DPO
& \textbf{66.64} / \textbf{19.10}
& {81.64} / \textbf{8.07}
& 70.13 / \textbf{7.62}
& \textbf{68.73} / 7.32
& {71.78} / {10.53} \\
& MASS\mbox{-}DPO
& {66.30} / 14.42
& \textbf{82.73} / 7.91
& \textbf{70.79} / 7.29
& {68.13} / \textbf{7.41}
& \textbf{71.99} / 9.26 \\
\midrule

\multirow{2}{*}{Llama3}
& S\mbox{-}DPO
& \textbf{83.44} / \textbf{23.93}
& \textbf{84.13} / \textbf{7.26}
& {70.58} / 6.42
& {63.92} / {5.20}
& {75.52} / \textbf{10.70} \\
& MASS\mbox{-}DPO
& {82.86} / {21.36}
& {84.06} / {7.24}
& \textbf{70.71} / \textbf{6.53}
& \textbf{65.58} / \textbf{5.75}
& \textbf{75.80} / {10.22} \\
\bottomrule
\end{tabular}
\label{tab:mrr-margin-all}
\end{table*}

\subsection{Computational Cost and Full-Pool Comparison}
\label{app:compute}

To quantify the computational trade-off between active selection and full-pool training, we measure wall-clock times on MovieLens with Llama-3.2-3B using 4$\times$H100 GPUs.

\begin{table}[h]
\centering
\caption{Wall-clock cost breakdown. MASS-DPO selection is a one-time preprocessing step; per-epoch training cost scales with the number of negatives.}
\label{tab:wallclock}
\small
\begin{tabular}{lc}
\toprule
Component & Wall-clock \\
\midrule
\quad Feature \& log-prob extraction & 42.0 min \\
\quad D-optimal subset selection (\Cref{alg:dopt}) & 3.5 min \\
MASS-DPO selection total (one-time) & 45.5 min \\
\midrule
MASS-DPO-5 training (per epoch) & 5.52 h \\
S-DPO-19 training (per epoch) & 10.21 h \\
\bottomrule
\end{tabular}
\end{table}

Training with the full negative pool (S-DPO-19) is 1.85$\times$ slower per epoch than MASS-DPO with $n{=}5$ selected negatives. The one-time selection cost of 45.5 min is dominated by feature and log-probability extraction (42 min); the core selection step (\Cref{alg:dopt}) itself takes only 3.5 min via rank-one updates. Over two epochs, S-DPO-19 incurs ${\sim}9.4$ extra GPU-hours relative to MASS-DPO-5, far exceeding the entire selection cost.

\Cref{tab:fullpool} compares MASS-DPO ($n{=}5$) at epoch 2 against S-DPO-all ($n{=}19$) at epoch 1 on MovieLens with Llama-3.2-3B at matched wall-clock budgets (${\sim}11$h vs.\ ${\sim}10.2$h).

\begin{table}[h]
\centering
\caption{MASS-DPO-5 vs.\ S-DPO-all at matched wall-clock budget on MovieLens + Llama-3.2-3B.}
\label{tab:fullpool}
\scriptsize
\begin{tabular}{llcccccccccc}
\toprule
Method & Epochs & Runtime (h) & Acc & R@1 & R@3 & R@5 & NDCG@1 & NDCG@3 & NDCG@5 & MRR & Margin \\
\midrule
MASS-DPO ($n{=}5$) & 2 & 11.03 & 0.604 & 0.606 & 0.810 & 0.870 & 0.606 & 0.726 & 0.751 & 0.725 & 5.049 \\
S-DPO-all (19 neg)  & 1 & 10.21 & 0.605 & 0.613 & 0.781 & 0.853 & 0.613 & 0.711 & 0.740 & 0.718 & 3.904 \\
\bottomrule
\end{tabular}
\end{table}

At matched compute, MASS-DPO-5 matches S-DPO-all on accuracy while outperforming it on R@3, R@5, NDCG@3, NDCG@5, and MRR, indicating that actively selected negatives provide more useful training signal per gradient step than the full pool.

\end{document}